\documentclass[sigconf]{acmart}

\usepackage{amsthm}
\usepackage{float}
\usepackage{bm}
\usepackage{geometry}
\usepackage{subfig}
\usepackage{caption}
\usepackage{makecell}
\usepackage{booktabs}
\usepackage{hyperref}
\usepackage{enumitem}
\usepackage{makecell}
\usepackage{xspace}
\usepackage{graphicx}
\usepackage{textcomp}
\usepackage{tcolorbox}
\usepackage{booktabs}
\usepackage{tabularx}
\usepackage{lipsum}
\usepackage{multirow}
\usepackage{color, xcolor}
\usepackage{balance}

\usepackage{breakurl}

\definecolor{light-gray}{gray}{0.8}

\AtBeginDocument{%
	\providecommand\BibTeX{{%
			\normalfont B\kern-0.5em{\scshape i\kern-0.25em b}\kern-0.8em\TeX}}}

\copyrightyear{2023}
\acmYear{2023}
\setcopyright{acmcopyright}\acmConference[WSDM '23]{Proceedings of the Sixteenth ACM International Conference on Web Search and Data Mining}{February 27-March 3, 2023}{Singapore, Singapore}
\acmBooktitle{Proceedings of the Sixteenth ACM International Conference on Web Search and Data Mining (WSDM '23), February 27-March 3, 2023, Singapore, Singapore}
\acmPrice{15.00}
\acmDOI{10.1145/3539597.3570457}
\acmISBN{978-1-4503-9407-9/23/02}

\begin{document}
	
	
        \title{MM-GNN: Mix-Moment Graph Neural Network towards Modeling Neighborhood Feature Distribution }
	
	
	\author{Wendong Bi}
        \authornote{Work performed during the internship at MSRA}
	\affiliation{%
		\institution{Institute of Computing Technology, Chinese Academy of Sciences}
		\city{Beijing}
		\country{China}}
	\email{biwendong20g@ict.ac.cn}
	
	\author{Lun Du}
	\authornote{Corresponding author}
	\affiliation{%
		\institution{Microsoft Research Asia}
		\city{Beijing}
		\country{China}}
	\email {lun.du@microsoft.com}
	
	\author{Qiang Fu}
	\affiliation{%
		\institution{Microsoft Research Asia}
		\city{Beijing}
		\country{China}}
	\email {qifu@microsoft.com}

        \author{Yanlin Wang}
	\affiliation{%
		\institution{Microsoft Research Asia}
		\city{Beijing}
		\country{China}}
	\email {yanlwang@microsoft.com}

        \author{Shi Han}
	\affiliation{%
		\institution{Microsoft Research Asia}
		\city{Beijing}
		\country{China}}
	\email {shihan@microsoft.com}

        \author{Dongmei Zhang}
	\affiliation{%
		\institution{Microsoft Research Asia}
		\city{Beijing}
		\country{China}}
	\email {dongmeiz@microsoft.com}

	\renewcommand{\shortauthors}{Wendong Bi et al.}
	
	\begin{abstract}
		Graph Neural Networks (GNNs) have shown expressive performance on graph representation learning by aggregating information from neighbors. Recently, some studies have discussed the importance of modeling neighborhood distribution on the graph. However, most existing GNNs aggregate neighbors' features through single statistic (e.g., mean, max, sum), which loses the information related to neighbor's feature distribution and therefore degrades the model performance. In this paper, inspired by the method of moment in statistical theory, we propose to model neighbor's feature distribution with multi-order moments. We design a novel GNN model, namely Mix-Moment Graph Neural Network (MM-GNN), which includes a Multi-order Moment Embedding (MME) module and an Element-wise Attention-based Moment Adaptor module. MM-GNN first calculates the multi-order moments of the neighbors for each node as signatures, and then use an Element-wise Attention-based Moment Adaptor to assign larger weights to important moments for each node and update node representations. We conduct extensive experiments on 15 real-world graphs (including social networks, citation networks and web-page networks etc.) to evaluate our model, and the results demonstrate the superiority of MM-GNN over existing state-of-the-art models
	\end{abstract}

	
	
	
    \settopmatter{printfolios=true}
	\maketitle
	
	\section{Introduction}
	Graph data is ubiquitous in the real world, such as social networks, citation networks, etc. Graph Neural Networks (GNNs) have shown expressive power on graph representation learning in recent years. Most existing GNNs follow the neighborhood aggregation scheme, where each node aggregates features of neighbors to update its own representation. Many GNNs \cite{GCN, GAT,derr2020epidemic,  yang2016revisiting}  with different aggregators have been proposed to solve various downstream tasks, which can be divided into graph-level tasks, edge-level tasks, and node-level tasks, such as graph classification \cite{GIN, bai2019simgnn}, link prediction \cite{zhang2018link, zhou2018dynamic, guo2022multi}, and node classification \cite{dai2022towards, lin2020initialization, lin2014large, fettal2022efficient, jin2021node, zhao2021graphsmote}.
	
	However, the aggregation schemes of existing GNNs usually use single statistics (e.g., mean, max, sum). For example, GCN \cite{GCN} uses mean aggregator, GraphSAGE \cite{GraphSAGE} uses max aggregator, and GIN uses sum aggregator. And some state-of-the-art GNNs (e.g. GCNII \cite{GCN2}, DAGNN \cite{DAGNN}) still use single statistics, without retaining the complete information of neighbor's feature distribution. Recently,  researchers \cite{GIN,du2021gbk, ma2022meta, luo2022ada} begin to study the modeling of neighborhood distribution and have discussed its importance  for graph data from various perspectives. However, the studies of neighbor's feature distribution are limited to lower-order characteristics without considering higher-order distribution characteristics.
	
	Considering that single statistic cannot represent complex distribution with more than one degree of freedom,  it's intuitive that we should introduce more higher-order statistics. To demonstrate the prominent effects of higher-order distribution characteristics on distinguishing different nodes, we present an example from Facebook Social Network \cite{facebook100} in Fig.~\ref{fig:1}.  Considering the graph composed of student/teacher nodes with class year as attributes, Fig.~\ref{fig:1} (b) shows the mean and variance of their neighbors' features. And we cannot distinguish the label of node A and node B by the mean of their neighbors' features. However, we can succeed to distinguish the two nodes if we consider the variance of their neighbors' features. And it's intuitional because the teachers are more likely to know students from different class years, thus the variance of  teacher's neighbors is larger. And this motivates us to consider higher-order characteristics of neighbor's feature distribution for graph data, and we further demonstrate the critical roles of  higher-order characteristics by analysis on real data in Section~\ref{section:findings}.
	
	\begin{figure}[h]
		\centering
		\includegraphics[width=0.9\linewidth, height=91pt]{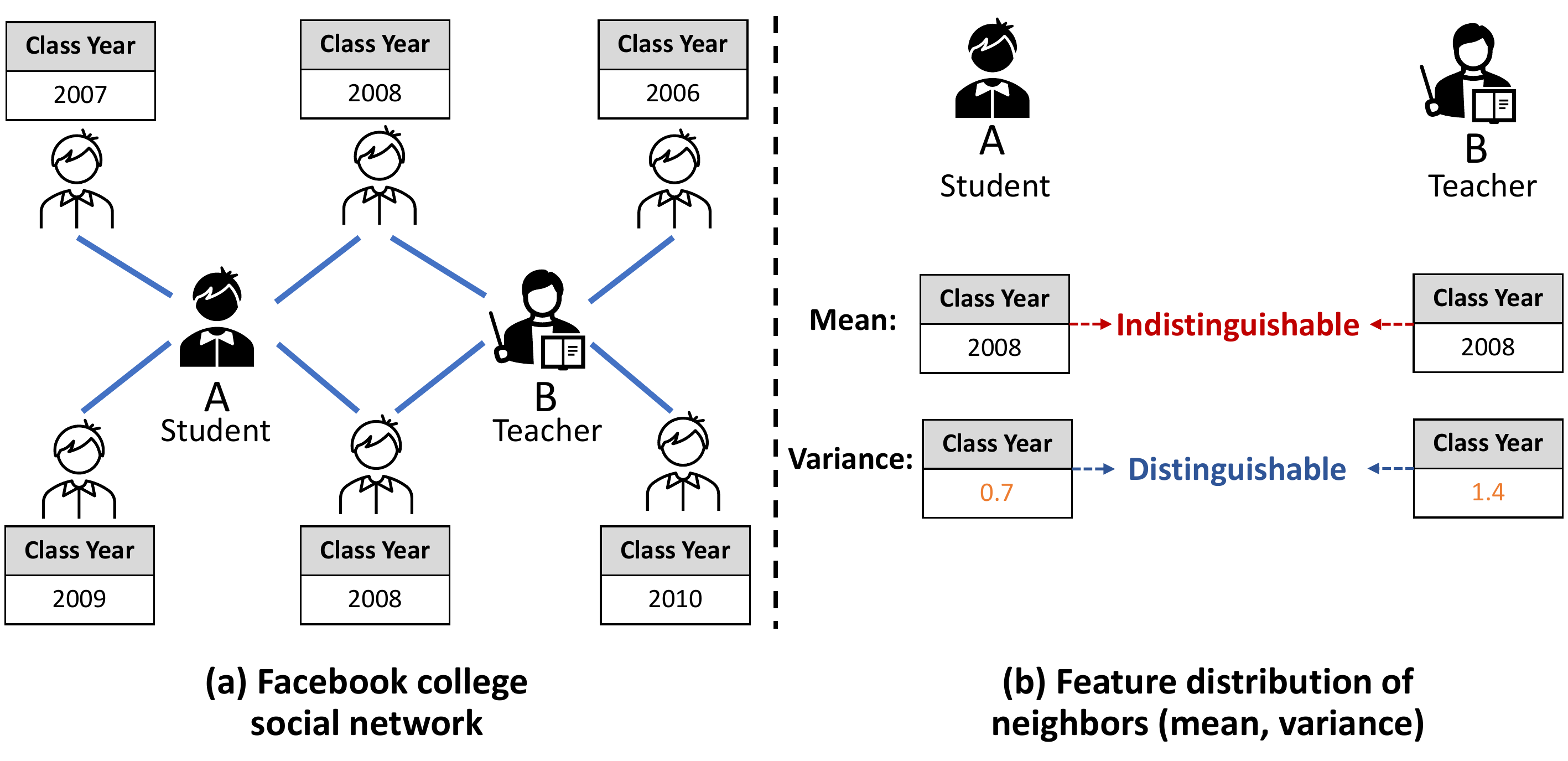}
		\caption{An example in Facebook social network. (a) is an example of Facebook college social network, where the student node A and the teacher node B have different neighbors set. (b) shows the feature distributions of neighbors of  A and  B by their mean and variance, where the mean information fails to distinguish the two nodes. 
		}
		\label{fig:1}
		\Description{A toy example of Facebook social network.}
	\end{figure}
	Inspired by the method of moments \cite{momentmethod}, which is used to represent unknown probability distributions, we introduce methods of moments to design a novel aggregator for GNNs.
	Specifically, we propose to use \textbf{adapted multi-order moments} to represent the neighbor's feature distribution and further embed the distribution signals into the message passing mechanism of GNNs. Then we design a novel  Mix-Moment Graph Neural Network (MM-GNN) model, which includes two major components: Multi-order Moments Embedding (MME) and Attention-based Moment Adaptor (AMA). MM-GNN first computes multi-order moments for neighbors of each node as their signatures. Based on the observations that different order of moments may have various impact for each node (Sec~\ref{section:findings}), an  Element-wise Attention-based  Moment Adaptor is designed to select important moment signatures for specific nodes adaptively. As it is impossible to compute infinite-orders of moments to represent the distribution perfectly, we only compute finite-order moments in parallel at each graph convolution year. Finally, we conduct extensive experiments  to demonstrate the superiority of MM-GNN over other state-of-the-art methods on 15 real-world graphs, including social networks, citation networks, webpage networks and image networks. The improved performance on all types of graphs demonstrate the effectiveness of our method.
	
	The main contributions of this paper are summarized as follows:
	\begin{enumerate}
		\item We demonstrate that higher-order statistics are needed to model neighbor's feature distribution on graphs. And we introduce the method of moments to GNNs from the perspective of modeling neighborhood feature distributions.
		\item We design a novel MM-GNN model with an aggregator based on the multi-order moments and further design an Element-wise Attention-based Moment Adaptor to adaptively adjust the weights of multi-order moments for each node.
		\item We conduct extensive experiments on 15 real-world  graphs (i.e., social networks, citation networks, webpage networks) to evaluate our method. The results show that MM-GNN gains consistent improvements on each type of graphs. 
	\end{enumerate}
	
	\section{Preliminaries}
	In this section, we  give the definitions of some important terminologies and concepts appearing in this paper.
	\subsection{Graph Neural Networks (GNNs)}
	GNNs aim to learn representation for nodes on the graph. Let $G=(V,E)$ denotes a graph, where $V=\{v_1, \cdots, v_n \}$ is the node set, $n=|V|$ is the number of nodes. Let $X\in \mathbb{R}^{n\times d}$ denote the feature matrix and the $l$-th row of $X$ denoted as $x_i$ is the $d$-dimensional feature vector of node $v_i$. $E$ is the edge set in $G$.  Typically, GNN-based model follows a neighborhood aggregation framework, where node representation is updated by aggregating information of its first or multi-order neighboring nodes. Let $h_i^{(l)}$ denotes the output vector of node $v_i$ at the $l$-th hidden layer and let $h_i^{(0)}=x_i$. The update function of $l$-th hidden layer is:
	\begin{equation}
	h_i^{(l)} = \text{COMBINE}\left(h_i^{(l-1)}, \text{AGGREGATE}\big(\{h_j^{(l-1)} | v_j \in \mathcal{N}(v_i) \}\big)\right),
	\end{equation}
	where $\mathcal{N}(v_i)$  is the set of neighbors of $v_i$.  The AGGREGATE function is aimed to gather information from neighbors and the goal of COMBINE function is to fuse the information from neighbors and the central node. For graph-level tasks, an additional READOUT function is required. 
\subsection{Method of Moments}
\label{section:moment_methods}
\subsubsection{Definitions about moments: } We give some basic definitions of statistical moments \cite{moment,momentmethod} in this section. First, the definition of moments is as follows:
    \begin{definition}[Moments in Mathematics]
		For a random variable $Z$, its $k$-th order \textbf{origin moment} is denoted as $\mathbb{E}(Z^k)$
		, where $\mathbb{E}(Z^k)=\int^\infty_{-\infty} z^k\cdot f(z;\theta_1, \cdots, \theta_k)dz$. The $k$-th order \textbf{central moment} of $Z$ is ${\small\mathbb{E}\left(\big(Z-\mu_z\big)^k\right)}=\int^\infty_{-\infty} (z-\mu_z)^k\cdot f(z;\theta_1, \cdots, \theta_k)dz$. The $k$-th order \textbf{standardized moment} of $Z$ is ${\small\mathbb{E}\left(\frac{\big(Z-\mu_z\big)^k}{\sigma_z^k}\right)}=\int^\infty_{-\infty} \frac{(z-\mu_z)^k}{\sigma_z^k}\cdot f(z;\theta_1, \cdots, \theta_k)dz$, where $\sigma_z$ is the standard deviation of random variable Z.
    \end{definition}
	Then, moments can be derived from the Moment-Generating Function (MGF), which is an alternative specification of its probability distribution. The MGF~\cite{moment} is defined as follows:
	\begin{definition}[Moment-Generating Function (MGF)]
		Let $X$ be a random variable with cumulative distribution function $F_X$. The Moment Generating Function (MGF) of $X$ (or $F_X$), denoted by $M_X(t)$, is $ M_X(t) = \mathbb{E}(e^{tX}) = \int_{-\infty}^{\infty}e^{tx}dF_X(x)$
	\end{definition}
	An essential property is that MGF can uniquely determine the distribution. The Taylor Expansion of MGF is as follows:
	\begin{equation}
	\label{eq:taylor}
	M_X(t) = \mathbb{E}(e^{tX}) = \sum_{k=0}^\infty \frac{t^k\mathbb{E}(X^k)}{k!} = \sum_{k=0}^\infty \frac{t^km_k}{k!}
	\end{equation}
	The $k$-th term of Eq.~\ref{eq:taylor} is exactly the $k$-th order origin moment.

	\subsubsection{Relationship between moments and common statistics:}
    \label{sec:rel_moment_statistic}
	In order to intuitively explain the physical meaning of multi-order moments, we give the relationship between moments  and some common distribution statistics of data. The common distribution statistics can be viewed as \textbf{special cases of moments}. For a random variable $Z$, the \textbf{expectation} of $Z$ is {\small$\mathbb{E}(Z)$}, which is the 1-st order origin moment. The \textbf{variance} of $Z$ is {\small$\mathbb{E}\left(\big(Z-\mu_z\big)^2\right)$}, which is  the 2-nd order central moment. The \textbf{skewness} of $Z$ is {\small$\mathbb{E}\left(\frac{\big(Z-\mu_z\big)^3}{\sigma_z^3}\right)$}, which is the 3-rd order standardized moment. 
	
	\section{Findings: neighbor distribution benefits node classification}
	
	\label{section:findings}

	In this section, We verify the importance of neighbor's feature distribution through comprehensive data analysis.
	We first select three common statistics (i.e., Mean, Variance, Skewness, see Sec.~\ref{sec:rel_moment_statistic}) to represent distribution of neighbor's feature. Then we calculate the selected statistics on the graphs of Facebook social network~\cite{facebook100}, 
	where the feature of nodes on the graphs have practical meanings (e.g., gender, major, year) and easy to analyze the physical meanings of their statistics. More details of the datasets are presented in Sec.~\ref{subsec:dataset}.
	And we further calculate the selected statistics on the neighbors' feature for all nodes and then take an average on all feature dimensions. Finally, we choose two metrics (i.e. Fisher Discrimination Index, Mutual Information) to 
 measure the correlation between  statistics of neighbor distribution
 and the node label, and one statistic is more important for node classification if it has stronger correlations with the node label.
	
	\subsection{Data Analysis on Correlation between Neighbor Distribution and Label}
	\begin{figure}[h]
		\begin{minipage}[t]{0.8\linewidth}
			\centering
			\subfloat[Fisher Discrimination Index.]{\includegraphics[width=\linewidth]{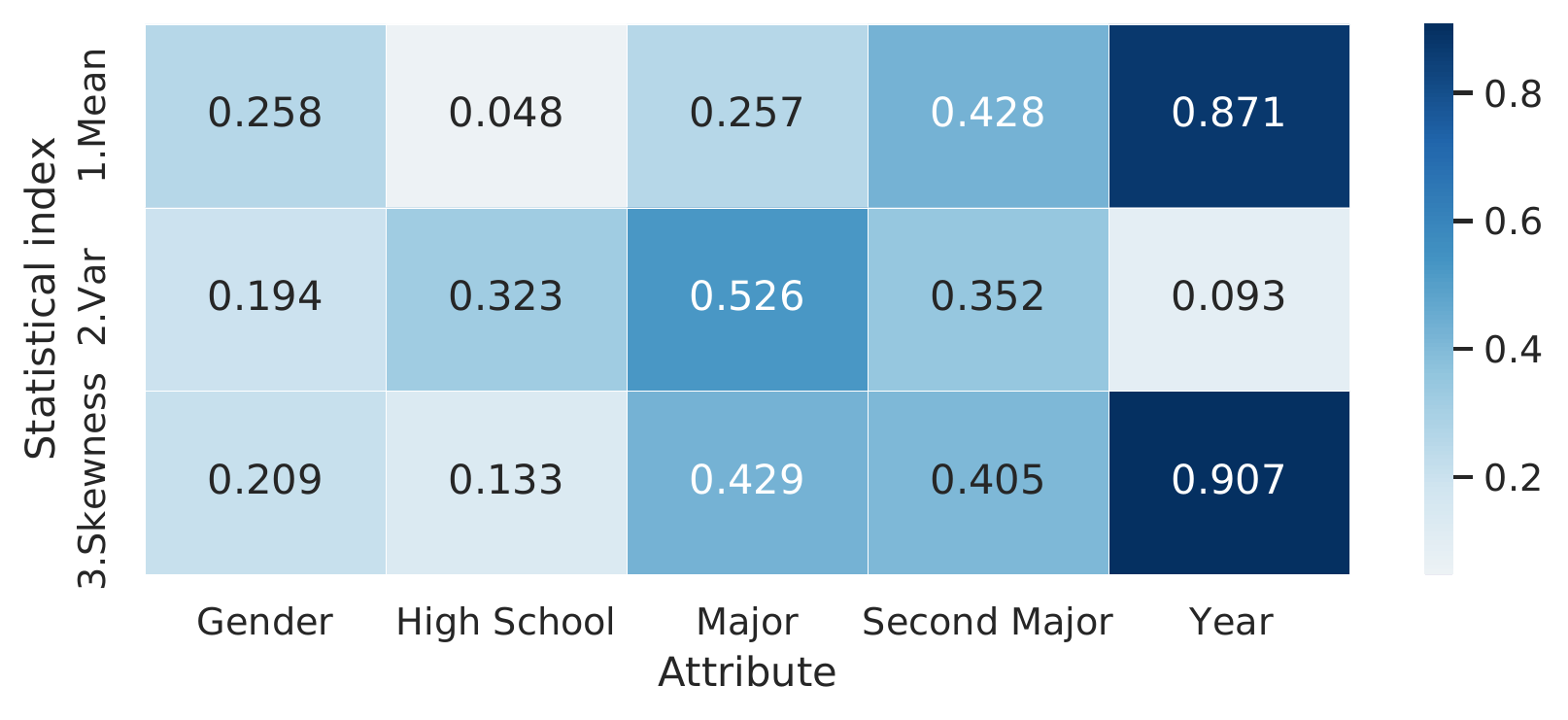}}
		\end{minipage}%
		\qquad\qquad
		\begin{minipage}[t]{0.8\linewidth}
			\centering
			\subfloat[Mutual Information.]{\includegraphics[width=\linewidth]{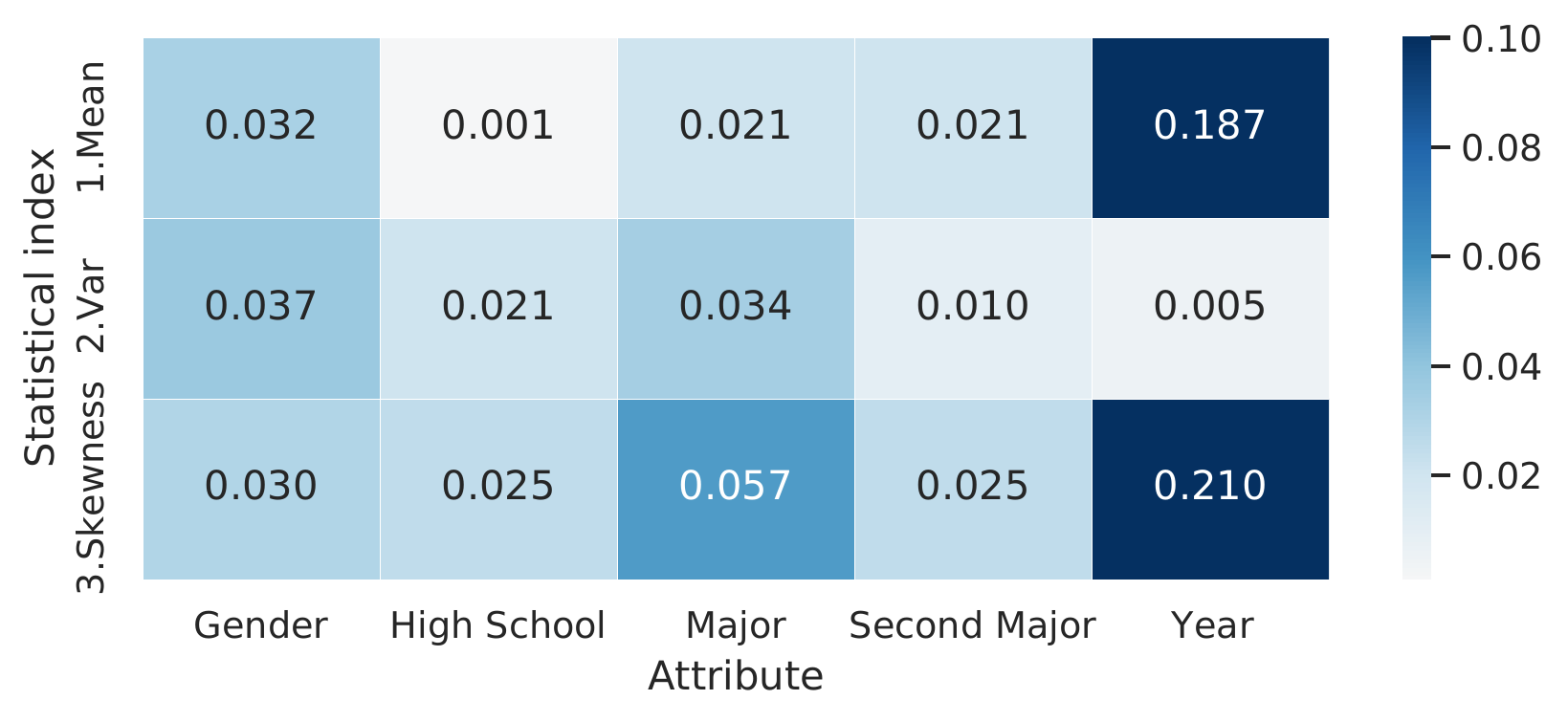}}
		\end{minipage}
		\caption{Correlation analysis between the node label and neighbor's feature distribution represented by statistics (i.e., Mean, Variance, Skewness) on the Facebook Social Networks.}
		
		\label{fig:heatmap_data}
	\end{figure}
	
	\subsubsection{Analysis by Fisher Discrimination Index}
	
	Inspired by Fisher's Linear Discrimination~\cite{fisher1936use}, we also use the Fisher discrimination index (Ratio of inter-class variance to intra-class variance) to measure the node classification capability distinguished by different statistics of neighbors' features distribution (e.g., Mean, Variance):
	\begin{equation}
    \small
    \nonumber
	S_{Fisher} = \frac{\sigma^2_{between}}{\sigma^2_{within}} =\sum_{\mathcal{C}_i, \mathcal{C}_j \in \mathcal{C} } \frac{\left(\mu_i - \mu_j\right)^2}{\sigma^2_i + \sigma^2_j},
	\end{equation}
	where $S_{Fisher}$ denotes the Fisher Discrimination Index, $\sigma^2_{between}$ and  $\sigma^2_{within}$ denote the inter-class and intra-class variation. $\mathcal{C}$ is the set of classes, $\mathcal{C}_i$ and $\mathcal{C}_j$ are two different categories, $\mu_i$ and $\sigma^2_i$ are respectively the mean and variance of samples in the class $\mathcal{C}_i$. As shown in Fig. \ref{fig:heatmap_data}~(a). We observe that, for different feature dimensions, different statistic of neighbor distribution contributes variously to distinguish node labels. For example, we observe that the skewness statistic has the best discriminative ability for the feature dimension "year", while the variance statistic has best discriminative ability for the feature dimension "major".

	\subsubsection{Analysis by Mutual Information}
	
	Apart from Fisher Discrimination Index, we also analyze the Mutual Information between distribution statistics and the node label to measure theirs correlations.
	As shown in Fig.~\ref{fig:heatmap_data}~(b), the results of mutual information also demonstrate that different statistics have different strength of correlation with the class label.  In summary, all these results show that different order of moments contribute quite differently to node classification. However, any single statistic (e.g., mean, max, sum) cannot reflect complete neighborhood feature distribution, we thus need to consider multi-order statistics such as variance, skewness and other higher-order distribution statistics. 
	
	\begin{figure*}[h]
		\centering
		\includegraphics[width=0.9\linewidth]{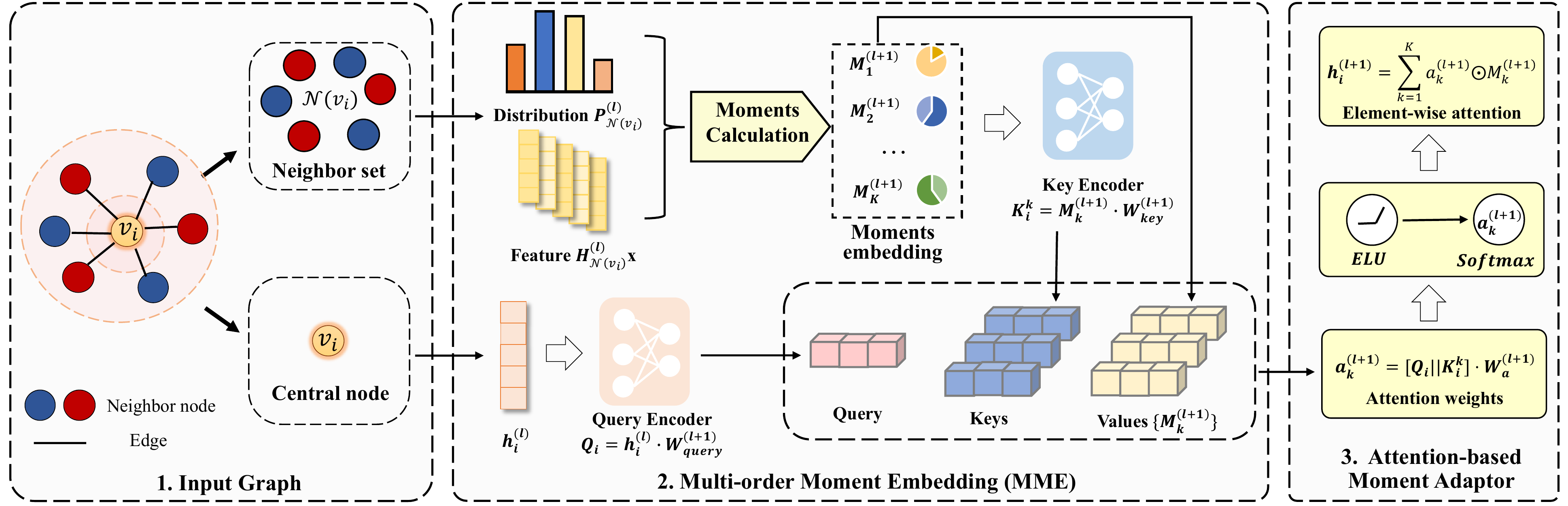}
		\caption{An overview of MM-GNN including Multi-order Moment Embedding (MME) and Attention-based Moment Adaptor (AMA) modules. {\small$H^{(l)}_{N(v_i)}$} and {\small$P^{(l)}_{N(v_i)}$} denotes the features and distributions (represented by moments) of neighbors at the $l$-th layer. {\small$h_i^{(l)}$} is the node representation of node $v_i$ at the $l$-th layer. $M_K^{(l+1)}$ denotes the $K$-th order moment of neighbors at the $(l+1)$-th layer. $Q_i$ and {\small$K_i^k$} denotes query and key vectors at the $l$-th layer. {\small$W_{query}^{(l+1)}$}, {\small$W_{key}^{(l+1)}$} and {\small$W_{query}^{(l+1)}$} are learnable transformation matrix at the $(l+1)$-th layer. {\small$a_k^{(l+1)}$} is the learned attention matrix for $k$-th order moment at the $(l+1)$-th layer.}
		\label{fig:model_structure}
		\Description{Figure of model structure}
	\end{figure*}
	


	\section{Mix-Moment Graph Neural Network}
	In this section, we introduce our Mix-Moment Graph Neural Network (MM-GNN) model. An overview of MM-GNN is given in the Fig.~\ref{fig:model_structure}. Compared with existing GNNs, our MM-GNN introduces the Method of Moments to better represent neighborhood feature distribution when updating node representations. 

	\subsection{Modeling Neighbors' Feature Distribution via Method of Moments}
	Based on the foundations of method of moments given in Sec.~\ref{section:moment_methods},  we introduce multi-order moments to represent the feature distribution of neighboring nodes on the graph.
	And based on the analysis in Sec.~\ref{section:findings}, we observe that, for the node classification task, the feature distribution of neighbors for the nodes belonging to the same class follow the similar distributions, and vice versa. Unlike other parametric methods, the outstanding power of estimating distribution with method of moments is that we do not need to assume an explicit form of the distribution function.   In this paper, we propose to use finite-order moments to represent the neighbors' feature distribution for a graph, and we further present the analysis on the theoretical and empirical deviation of representing the neighbors' feature distribution for graph data with finite-order moments in Sec~\ref{sec:taylor_error}.
	
	\subsection{Multi-order Moments Embedding (MME)}
	As aforementioned, we model neighbors' feature distribution by introducing the Method of Moments to GNN. The Method of Moments illustrated in Section \ref{section:moment_methods} is a statistical method that use sample moments to estimate the distribution function.
    Therefore, we use multi-order moments to estimate the feature distribution of neighboring nodes. Moreover, moments of different nodes can be computed in parallel at the aggregation step, and we select both origin moment and central moment to represent neighborhood feature distribution. Here we present the form of $k$-th order origin moment for $(l+1)$-th layer of MM-GNN:
	\begin{equation}
	\label{equation:origin_moment}
	M_k^{(l+1)}(v_i) = \mathbb{E}\big((h_j^{(l)})^k\big) = \left( \frac{1}{|\mathcal{N}(v_i)|}\sum_{j\in \mathcal{N}(v_i)} ((h_j^{(l)})^k \right)^{\frac{1}{k}} \cdot W_k^M
	\end{equation}
	where $W_k^M \in \mathbb{R}^{D_{in}\times D_{hidden}}$ is the transformation matrix for $k$-th order moment signature, $h_j^{(l)}\in\mathbb{R}^{D_{in}}$ is the node representation of  $v_j$ output by the $l$-th layer, and $h_j^{0}$ is the raw node feature $x_j$. Based on the Eq.~\ref{equation:origin_moment}, we can observe that the $1$-st origin moment is exactly the expectation, which is equivalent to the MEAN aggregation strategy. Thus previous GNNs with MEAN aggregator are special cases of MM-GNN using $1$-st order origin moment.
	
	Similarly, we can give the form of $k$-th order central moment for $(l+1)$-th layer of GNNs:
	\begin{equation}
	M_k^{(l+1)}(v_i) = \mathbb{E}\big((h_j^{(l)} - \mu_i^{(l)})^k\big) = \left( \frac{1}{|\mathcal{N}(v_i)|}\sum_{j\in \mathcal{N}(v_i)} (x_j - \mu_i^{(l)})^k \right)^{\frac{1}{k}}\cdot W_k^M
	\end{equation}
	, where
	\begin{equation}
	\mu_i^{(l)}  = \frac{1}{|\mathcal{N}(v_i)|}\sum_{j\in \mathcal{N}(v_i)} h_j^{(l)}
	\end{equation}
	To keep the same order of magnitude, we normalize moments of different orders by computing the $\frac{1}{k}$ power of $k$-th order moment vectors in the model. Besides, we set a hyper-parameter for flexible model configuration of using the central moment or origin moment.

	\subsection{Attention-based Moment Adaptor (AMA)}
	With the Multi-order Moments Embedding (MME) module, we get signatures representing moments of different orders for each node. Then, we need to fuse these signatures into one final representation for each node. As illustrated in Section \ref{section:findings}, different statistics (corresponding to moments) have different effects on node classification, we design an Element-wise Attention-based Moment Adaptor (AMA) that can  adaptively select the important order of moments for different nodes and feature dimensions. For each MM-GNN layer, after we get the multi-order moment signatures with the MME module, we use the transformed input node representation (the output of the previous layer) as the query vector for each node, and then we use different moment-signatures (output of the MME module of current layer) as key vectors. Then we compute attentions as weights of different moments for each individual nodes. The attention of $k$-th order moment for $v_i$ at $(l+1)$-th layer, denoted by $a_k^{(l+1)}(v_i)\in \mathbb{R}^{D_{hidden}}$, is:
	\begin{equation}
	a_k^{(l+1)}(v_i) = \sigma\left([h_i^{(l)} \cdot W_{query}^{(l+1)} || M_k^{(l+1)}(v_i) \cdot W_{key}^{(l+1)} ]\cdot W_a^{(l+1)}  \right)
	\end{equation}
	where $W_a^{(l)}\in\mathbb{R}^{2D_{hidden}\times D_{hidden}}$ is the learnable attention matrix in $l$-th layer, $W_{query}^{(l)}, W_{key}^{(l)}\in\mathbb{R}^{D_{hidden}\times D_{hidden}}$  are the transformation matrix for the query and key vectors in $l$-th layer. Then different dimensions of $k$-th order moment signature for node $v_i$ use different attention weights and therefore implement an element-wise attention mechanism (attention of each order moment is a $D_{hidden}$-dimensional vector rather than a scalar). Then we can get the output representation of node $v_i$ at $(l+1)$-th layer:
	\begin{equation}
	h_i^{(l+1)} = \sum_{k=1}^K a_k^{(l+1)}(v_i) \odot  M_k^{(l+1)}(v_i) ,\ a_k^{(l+1)}(v_i)\in\mathbb{R}^{1\times D_{hidden}}
	\end{equation}
	where $\odot$ is the element-wise product and $K$ is the hyper-parameter for the largest ordinal of moments used in MM-GNN. Besides, we add residual connections across layers to further boost the performance of MM-GNN by aggregating neighborhood feature distribution of multi-hop neighbors. We finally use the output of the last GNN layer as the node representation  and further compute Cross Entropy Loss for node classification  task and back propagate to optimize the model parameters.

	\begin{table*}
        \small
		\setlength{\tabcolsep}{6.0pt}
		\caption{Information of Facebook social networks. We present the number of nodes, number of edges for each social graph in this table. Each node has 6-dimensional attributes and belongs to 6 possible classes.}
		\label{tab:data_info}
		\begin{tabular}{cccccccccc}
			\toprule
			Dataset & Northeastern & Caltech & UF & Hamilton & Howard & Simmons & GWU & UGA & Tulane  \\ 
			\midrule
			\#Nodes & 13882 & 769 & 35123 & 2314 & 1909 & 1518 & 12193 & 24389 & 7752 \\
			\#Edges& 763868 & 33312& 2931320 & 192788 & 409700 & 65976& 939056 &  2348114 & 567836 \\
			\bottomrule
		\end{tabular}
	\end{table*}
	\begin{table}
        \small
		\caption{Information of citation/webpage/image networks}
		\setlength{\tabcolsep}{6.0pt}
		\label{tab:citation_info}
		\begin{tabular}{ccccc}
			\toprule
			Dataset & \#Nodes & \#Edges & \#Features & \#Classes  \\
			\hline
			Cora & 2,708 & 10,556 & 1,433& 7\\
			CiteSeer & 3,327&9,104 &3,703 &6 \\
			PubMed &19,717 &88,648 &500 & 3\\
			Chameleon & 2277 & 36101 & 2325 & 5 \\
			Squirrel & 5201 & 217073 & 2089 & 5 \\
			Flickr & 89,250 & 899,756 & 500 & 7 \\
			\bottomrule
		\end{tabular}
		
	\end{table}
	
	\section{Deep Analysis of MM-GNN}
	\label{sec:theory}
	In this section, we first theoretically analyze the limitations of existing GNNs using single statistic, which cannot retain the full information of neighbor distribution. Then we analyze the deviation of estimating neighbors' feature distribution with finite-order moments. Finally, we analyze the time complexity of MM-GNN.
    
    \subsection{Limitations of Existing GNNs}
	\label{section:drawback_gnn}
    In this sections, we prove the existing GNNs using single statistic (e.g. GCN, GraphSAGE) have severe limitations when representing neighbor's feature distribution, which reduces the generalization ability of the model. Complexity measure \cite{neyshabur2017exploring} is the current mainstream method to measure the generalization ability of the model. And we choose Consistency of Representations \cite{complexity} (champion of the NIPS2020 Competition on generalization measure). And the complexity measure of the model, denoted as $\Gamma$,  is:
    \begin{equation}
    \label{eq:complexity}
        \Gamma = \frac{1}{|\mathcal{C}|}\sum_{i=1}^{|\mathcal{C}|} \max_{i\neq j}\frac{\mathcal{S}_i + \mathcal{S}_j }{\mathcal{M}_{i,j}},
    \end{equation}
    where $\mathcal{C}$ is the set of classes, $\mathcal{C}_i, \mathcal{C}_j \in \mathcal{C}$  are two different classes, $\mathcal{S}_i=\left(\mathbb{E}_{v_k\sim\mathcal{C}_i}(|h_k - \mu_{\mathcal{C}_i}|^p)\right)^{\frac{1}{p}}$ is the intra-class variance of class $\mathcal{C}_i$,  $\mathcal{M}_{i,j}=||\mu_{\mathcal{C}_i} - \mu_{\mathcal{C}_j}||_p$ is the inter-class variance between $\mathcal{C}_i$ and $\mathcal{C}_j$, $h_k$ is the representations of nodes $v_k$ learned by the model and $\mu_{\mathcal{C}_i} = \mathbb{E}_{v_k\sim\mathcal{C}_i}(h_k)$. And \textbf{higher complexity measure means lower generalization ability}. Then we demonstrate that GNNs using single statistic have weak generalization ability. Take GCN~\cite{GCN} as an example, we give the following theorem:
    \begin{theorem}
    \label{theorem:limitation}
        Given a graph $G(V,E)$, we denote the nodes belonging to class $\mathcal{C}_i$ as $\{v_i | v_i\in \mathcal{C}_i\}$, and assume that the neighbor's feature distribution of nodes in $\mathcal{C}_i$ follows the same i.i.d Gaussian distribution $\mathrm{N}(\bm{\mu}_i, \bm{\Sigma}_i)$. Given any two classes $\mathcal{C}_i\neq \mathcal{C}_j$, if $\bm{\mu}_i = \bm{\mu}_j$ and $\bm{\Sigma}_i\neq \bm{\Sigma}_j$ , then the complexity measure $\Gamma$ (Eq.\ref{eq:complexity}) of GCN with mean aggregation reaches $+\infty$, and the GCN thus lose the generalization ability.
    \end{theorem}
	\begin{proof}
        The update function of GCN with mean aggregator is:
        \begin{equation}
        \nonumber
	h^{(l+1)}_k  =\sum_{j\in \mathcal{N}(v_k)}  \frac{1}{d_k}\cdot h_j^{(l)} \cdot W^{(l)} = \left(\frac{1}{d_k}\cdot\sum_{j\in \mathcal{N}(v_k)}   h_j^{(l)}\right) \cdot W^{(l)}
        \end{equation}
        To simplify the analysis, we omit the activation function. And we can get the expectation of $h^{(l+1)}_k$ for nodes in $\mathcal{C}_i$ and $\mathcal{C}_j$:
        \begin{equation}
        \label{eq:equal_expectation}
            \mathbb{E}_{v_k\sim\mathcal{C}_i}(h_k^{(l+1)}) = \bm{\mu}_i\cdot W^{(l)} = \mathbb{E}_{v_k\sim\mathcal{C}_j}(h_k^{(l+1)}) = \bm{\mu}_j\cdot W^{(l)} ,
        \end{equation}
        which means the expectation of the learned representations of nodes in $\mathcal{C}_i$ equals to that of nodes in $\mathcal{C}_j$. Then according to Eq.~\ref{eq:equal_expectation}:
        \begin{equation}
            \nonumber
        \mathcal{M}_{i,j} = ||\mathbb{E}_{v_k\sim\mathcal{C}_i}(h_k) - \mathbb{E}_{v_k\sim\mathcal{C}_j}(h_k) ||_p = \bm{0}
        \end{equation}
        Considering that $\mathcal{S}_i + \mathcal{S}_i > 0$, the complexity measure $\Gamma$ in Eq.~\ref{eq:complexity} equals to $+\infty$. Then GCN loses generalization ability, and this rule also applies to any other GNNs with mean aggregator.
	\end{proof}
    And we can intuitively generalize Theorem~\ref{theorem:limitation} to all GNNs using single statistics. For example, GraphSAGE~\cite{GraphSAGE} use the max/mean aggregator, GIN~\cite{GIN} uses the sum aggregator. And most state-of-the-art GNNs (e.g., GCNII~\cite{GCN2}, DAGNN \cite{DAGNN}) still use single statistic to represent the neighbor's feature distribution , and therefore have the same limitations when neighbor's feature distribution has more than one degree of freedom.
	
    \subsection{Deviation Analysis on Finite-order Moments.}
	\label{sec:taylor_error}
     We give the theoretical guarantee of estimating neighborhood feature distributions on the graph with finite-order moments by a theorem and give its proof as follows:
	\begin{theorem}
		\label{theorem:finite-order-moments}
         Given a graph $G$, we denote the deviation  of using top-$K$ order moments to represent the neighbor's feature distribution as $e_K$, and $e_K$ wil not larger than an upper bound: $e_K \leq \frac{(\epsilon\cdot c)^{k+1}}{(k+1)!}$,
      where $c\in \mathbb{R}^+$ is a constant and $\epsilon\in \mathbb{R}$ is a small number around 0. And larger $K$ leads to smaller estimation deviation (convergence to 0).
	\end{theorem}
	\begin{proof}
		According to the Sec.~\ref{section:moment_methods}, the $k$-th order moment can be derived from the $k$-th term of the MGF's Taylor Expansion (see eq.~\ref{eq:taylor}). Therefore, the top-$K$ order moments can be viewed as an estimation of the $K$-th order Taylor expansion for MGF of neighbor's feature distribution. In real cases, the  feature of samples is usually bounded. We assume that for $\forall x_{i,j} \in X$, we have $|x_{i,j}| \leq c$, where $c\in \mathbb{R}^+$ is a constant. Besides, sine the Eq.~\ref{eq:taylor} is the Taylor expansion around $t=0$, then $t$ is a small real number $t\leq \epsilon$, where $\epsilon \in \mathbb{R}$ is a small number around 0. Then the deviation of estimating neighbor's feature distribution by  $k$-th order Taylor expansion of MGF can be approximated by its Lagrange Remainder:
		\begin{equation}
		R_k(x) = \frac{t^{k+1}\mathbb{E}(X^{k+1})}{(k+1)!}  \leq \frac{(t\cdot c)^{k+1}}{(k+1)!} \leq \frac{(\epsilon\cdot c)^{k+1}}{(k+1)!}
		\end{equation}
		Besides, in real scenarios, we usually normalize the data, so the upper bound of the estimation error $\frac{(t\cdot c)^{k+1}}{(k+1)!} $ is further reduced. And it's obvious that the upper bound of estimation error converges to zero as the moment ordinal $k$ increases.
	\end{proof}

	\subsection{Time Complexity Analysis}
	Here we analyze the time complexity of training MM-GNN. Let $D_{in}$ denote the dimension of input node feature, and $D_{hidden}$ denote the hidden dimension of node representation at each layer. $|V|$ and $|E|$ denotes the number of nodes on the graph. For each layer of MM-GNN, the aggregation and transformation of graph convolution cost $O\left((|V|+|E|)\cdot D_{in}D_{hidden}\right)$. Let $k$ denote the max ordinal of used moments, and we used $k\leq 3$ in our experiments and thus can be seen as a small constant integer. The calculation of query vector costs $O(|V|D_{hidden}^2)$, and the calculation of all key vectors costs $O\left(k\cdot (|V|+|E|)\cdot D_{hidden}^2\right)$. The Attention-based Moment Adaptor costs $O(k|V|D_{hidden}^2)$. Considering that $D_{hidden} \leq D_{in}$ and $k$ is a constant integer, the overall time complexity of MM-GNN layer is:
	\begin{equation}
	\nonumber
	\begin{aligned}
	T &= O\left( (k+1)|V|D_{hidden}^2 + (k+1)\cdot (|V|+|E|)\cdot D_{in} D_{hidden}\right) \\
	&= O\left((|V|+|E|)\cdot D_{in} D_{hidden}\right)
	\end{aligned},
	\end{equation}
	which equals to the complexity of other mainstream GNNs. 

	\section{Experiments}
	\label{section:experiment}
	\subsection{Datasets}
	\label{subsec:dataset}
	The proposed MM-GNN is evaluated on 15 real-world graphs, including 9 social graphs from Facebook social networks, 3 citation networks (Cora, CiteSeer, PubMed), 2 webpage networks (Chameleon, Squirrel) and one image network (Flickr). The detailed information of theses  datasets are presented in Table~\ref{tab:data_info} and Table~\ref{tab:citation_info}.
	
	\textbf{Social network datasets: }
	Facebook100~\cite{facebook100} provides 100 college graphs, each graph describes the social relationship in a university. Each graph has categorical node attributes with practical meaning (e.g., gender, major, class year.). Moreover, nodes in each dataset belong to six different classes (a student/teacher status flag). 
	
	\textbf{Citation network datasets: }
	We choose three public citation network datasets (Cora, PubMed, CiteSeer) to evaluate our model. The nodes on such graphs represent papers and edges represent citation relationships between papers.
	
	\textbf{Webpage network datasets: }
	Two webpage network datasets (Chameleon, Squirrel) are used for model evaluation.  Nodes of these datasets represent articles and edges are mutual links between them. 
	
	\textbf{Image network datasets: }
	We also choose one image network (Flickr) for evaluation. The nodes represent images and edges means two images share the same metadata (e.g., location, gallery, tags).

	\begin{table*}
		\small
		\setlength{\tabcolsep}{1.0pt}
		\caption{Performance (accuracy) on node classification task ($60\%$ training set ratio) on Facebook social network datasets.}
		\label{tab:res1}
		\begin{tabular}{cccccccccc}
			\toprule
			Model & Northeastern & Caltech & UF & Hamilton & Howard & Simmons & GWU & UGA& Tulane \\
			\midrule
			GCN & $90.58\pm 0.28$ & $88.47\pm 1.91$ & $83.94\pm 0.61$ & $92.26\pm 0.35$ & $91.21\pm 0.69$ & $88.74\pm 0.61$ & $86.66\pm 0.48$ & $85.73\pm 0.51$ & $87.93\pm 0.97$\\
			GAT & $88.69\pm 0.30$ & $81.17\pm 2.15$ & $81.68\pm 0.59$ & $91.43\pm 1.25$ & $89.35\pm 0.44$ & $89.14\pm 0.57$ & $84.27\pm 0.88$ & $82.95\pm 0.75$ & $84.45\pm 1.40$\\
			GraphSAGE & $91.77\pm 0.86$ & $92.14\pm 1.38$ & $86.21\pm 0.56$ & \underline{$94.75\pm 0.35$} & $93.13\pm 0.93$ & $92.31\pm 0.92$ & $89.43\pm 0.78$ & $88.11\pm 0.39$ & \underline{$89.87\pm 0.47$}\\
			GIN  & $89.62\pm 1.10$ & $78.35\pm 4.24$ & $85.40\pm 0.71$ & $85.47\pm 6.20$ & $89.37\pm 2.05$ & $88.37\pm 1.18$ & $84.73\pm 2.80$ & $87.68\pm 0.40$ & $80.71\pm 4.78$\\
			APPNP  & $91.11\pm 0.27$ & $90.76\pm 2.38$ & $83.07\pm 0.54$ & $93.99\pm 0.47$ & $91.88\pm 1.22$ & $90.31\pm 1.20$ & $90.43\pm 0.38$ & $86.31\pm 0.41$ & $88.52\pm 0.44$\\
			JKNet  & $92.32\pm 0.39$ & \underline{$92.34\pm 1.09$} & $85.43\pm 0.71$ & $94.72\pm 0.76$ & $93.11\pm 0.78$ & $91.81\pm 0.71$ & $89.25\pm 0.65$ & $87.68\pm 0.40$ & $89.04\pm 0.43$\\
			DAGNN  & \underline{$92.43\pm 0.30$}& $91.15\pm 2.67$ & \underline{$86.71\pm 0.51$} & \underline{$94.75\pm 0.75$} & \underline{$93.44\pm 0.68$} & \underline{$92.49\pm 0.96$} & \underline{$89.83\pm 0.42$} & \underline{$88.20\pm 0.29$} & $89.29\pm 0.50$
			\\
			MM-GNN & \colorbox{light-gray}{\bm{$93.16\pm 0.32$}} & \colorbox{light-gray}{\bm{$93.28\pm 1.94$}} & \colorbox{light-gray}{\bm{$88.20\pm 0.45$}} & \colorbox{light-gray}{\bm{$95.46\pm 0.53$}} & \colorbox{light-gray}{\bm{$94.08\pm 0.95$}} & \colorbox{light-gray}{\bm{$93.61\pm 0.45$}} & \colorbox{light-gray}{\bm{$90.78\pm 0.45$}} & \colorbox{light-gray}{\bm{$89.44\pm 0.30$}} & \colorbox{light-gray}{\bm{$90.22\pm 0.33$}}
			\\
			\bottomrule
		\end{tabular}
		
	\end{table*}

	\begin{table*}
		\small
		\setlength{\tabcolsep}{6.0pt}
		\caption{Performance (accuracy) on node classification task on citation/webpage/image network datasets}
		\label{tab:res_citation}
		\begin{tabular}{ccccccc}
			\toprule
			Dataset & Cora & CiteSeer & PubMed & Chameleon & Squirrel & Flickr  \\
			\hline
			GCN & $81.13\pm0.26$ & $70.08\pm0.43$ & $78.62\pm0.63$ & $37.68\pm3.06$ & $26.39\pm0.88$ & $50.31\pm0.30$ \\
			GAT & $81.36\pm0.96$ & $70.93\pm0.70$ & $78.19\pm0.78$ & $44.34\pm1.42$ & $29.82\pm0.98$  & $50.59\pm0.26$ \\
			GraphSAGE & $80.43\pm0.39$ & $69.56\pm0.29$ & $77.47\pm0.76$ & $47.06\pm1.88$ & $35.62\pm1.21$  & $50.21\pm0.31$ \\
			GIN & $76.77\pm0.88$ & $67.67\pm0.34$ & $77.03\pm0.53$ & $32.18\pm1.98$ & $25.08\pm1.36$  & $46.65\pm0.33$  \\
			APPNP & $82.57\pm0.10$ & $70.53\pm0.25$ & $78.33\pm0.39$ & $40.44\pm2.02$ & $29.20\pm1.45$  & $49.21\pm0.37$ \\
			JKNet & $79.43\pm0.53$ & $70.36\pm 0.33$ & $77.80\pm0.82$ & $42.43\pm1.76$ & $35.52\pm1.06$  & $49.83\pm0.58$ \\
			DAGNN & \underline{$83.93\pm0.63$} & \underline{$72.56\pm1.76$} & \colorbox{light-gray}{\bm{$80.50\pm0.75$}} & {$47.29\pm0.62$} & $36.23\pm1.58$  & \underline{$50.78\pm0.62$}  \\ 
			GPR-GNN  & $81.85 \pm 0.38$ & $70.64 \pm 0.65$ & $79.78 \pm 0.43$ & \underline{$60.83 \pm 1.37$} & \underline{$49.21 \pm 1.15$} & $50.01\pm 0.53$ \\
			MM-GNN & \colorbox{light-gray}{\bm{$84.21\pm0.56$}} & \colorbox{light-gray}{\bm{$73.03\pm0.58$}} & $\underline{80.26\pm0.69}$ & \colorbox{light-gray}{\bm{$63.32\pm 1.31$}} & \colorbox{light-gray}{\bm{$51.38\pm 1.73$}}   &  \colorbox{light-gray}{$\bm{51.73\pm0.35}$}\\ 
			
			\bottomrule
		\end{tabular}
		
	\end{table*}
	
	\subsection{Baselines}
	We compare our proposed MM-GNN model against eight widely-used GNN models, including four baseline GNNs ( GCN \cite{GCN}, GAT \cite{GAT}, GraphSAGE \cite{GraphSAGE}, GIN \cite{GIN}) and four state-of-the-art GNNs (APPNP \cite{APPNP}, JKNet \cite{JKNet}, DAGNN \cite{DAGNN}, GPRGNN \cite{gprgnn}). Considering that the two webpage networks (Chameleon, Squirrel) are heterophily graph datasets unlike other datasets used in this paper, we also choose GPR-GNN~\cite{gprgnn}, which designed for graphs with heterophily as one of the baseline models.
	

	\subsection{Experimental Setup}
	We evaluate MM-GNN on the semi-supervised node classification task compared with state-of-the-art methods. For each dataset, we all run our model and other baseline models for 10 times and compute the mean and standard deviation of the results.  For Facebook social networks, we randomly generate 3 different data splits with an average train/val/test split ratio of 60\%/20\%/20\% in the Table~\ref{tab:res1}. And we further generate data splits with lower training ratio (40\%,20\%,10\%) and conduct experiments to validate the robustness of our model in the Table~\ref{tab:res2}. For citation networks (Cora, Citeseer, and Pubmed), we use the public split recommended by \cite{GCN} with fixed 20 nodes per class for training, 500 nodes for validation, and 1000 nodes for testing. For webpage networks (Chameleon, Squirrel) \cite{pei2020geom} and image network (Flickr \cite{zeng2019graphsaint}), we all use the public splits recommended in the original papers.
    The max moment ordinal $K$ for MM-GNN is searched in \{1, 2, 3\} in our experiments.	To keep fairness, we search the shared hyper-parameters for all models in the same searching spaces. Specifically, we train 200/400/600 epochs for each model. we search the hidden units in \{16, 32, 64, 128, 256\} and search the number of GNN layers in \{2, 6, 10, 16\}. We search learning rate from $1e-4$ to $1e-1$ and the weight decay from $1e-4$ to $1e-2$. For each experiment, we use the Adam optimizer \cite{kingma2014adam} to train each GNN model on Nvidia Tesla V100 GPU.

	\subsection{Main Results}
	\label{section:results}
	We compare the performance of MM-GNN to the state-of-the-art methods in Table~\ref{tab:res1} and Table~\ref{tab:res_citation}. Compared with all baselines, MM-GNN generally achieves the best performance on all datasets.  For social networks, we observe from Table~\ref{tab:res1} that our MM-GNN model gains significant improvement on all datasets. Besides, MM-GNN also gain significant improvements on other public real-world datasets, including citation networks, webpage networks and image networks. And the results demonstrate that MM-GNN gains superior performance on the graphs of various types.

	\subsection{Ablation Study}
	
	To estimate the effectiveness of each component in MM-GNN, we conduct comprehensive ablation studies by removing certain component at a time. The results are shown in the Table \ref{tab:ablation}.
	
	\begin{table}[h]
		\small
            \setlength{\tabcolsep}{3pt}
		\caption{Performance (accuracy) on node classification task compared with single-moment GNN models. M-1, M-2, M-3 represent the single-moment GNN models, where M-$k$ denotes the model whose aggregator only uses the $k$-th order moment of neighbors. Ensemble denotes the mean ensemble of the three single-moment models. }
		\label{tab:ablation}
		\begin{tabular}{c|ccc|c|cc}
			\toprule
			Model & M-1 & M-2 & M-3 & Ensemble & \multicolumn{2}{c}{MM-GNN }   \\
			Fusion mode & --- & --- & --- & Mean & MLP & Attention  \\
			\hline
			Howard & 93.15 & 93.61 & 93.83 & 93.70 & \underline{94.00} & \colorbox{light-gray}{\textbf{94.37}} \\
			Simmons & 90.90 & 92.36 & 92.45 & 91.85 & \underline{92.99} & \colorbox{light-gray}{\textbf{93.66}} \\
			GWU & 88.57 & 89.79 & 90.44 & 89.69 & \underline{90.48} &  \colorbox{light-gray}{\textbf{91.11}} \\
			Cora & 82.31 & 81.70 & 80.07 & 81.98 & \underline{82.58} &  \colorbox{light-gray}{\textbf{84.21}} \\
			CiteSeer & 71.31 & 70.70 & 71.53 & 71.28 & \underline{72.33} &  \colorbox{light-gray}{\textbf{73.03}} \\
			PubMed & 79.23 & 78.01 & 78.27 & 78.93 & \underline{79.35} &  \colorbox{light-gray}{\textbf{80.26}} \\
			Chameleon & 48.17 & 48.09 & 48.24 & 49.76 & \underline{50.05} & \colorbox{light-gray}{\textbf{51.93}} \\
			Squirrel & 35.58 & 36.29 & 35.97 & 36.09 & \underline{36.35} &  \colorbox{light-gray}{\textbf{38.92}} \\
			Flickr & 50.37 & 51.09 & 50.88 & 51.15 & \underline{51.37} & \colorbox{light-gray}{\textbf{51.73}} \\
			\bottomrule
		\end{tabular}
	\end{table}

	\subsubsection{Effects of MME} 
	To validate the effects of Multi-order Moments Embedding (MME) module, we design several model variants as shown in Table \ref{tab:ablation}. The variants in the first three columns use only one moment (e.g. 1st-order moment) for neighborhood aggregation. The Ensemble-Mean in the table is another variant that ensemble all single-moment models with the average logits of all single-moment variants.  And we observe that each single-moment model has different performance on different datasets. And the simple ensemble of them usually have better performance. However, our full MM-GNN (the  rightmost column) always has the best performance on all datasets against all variants.

	\subsubsection{Effects of AMA}
	To validate the effects of the Attention-based Moment Adaptor (AMA) module, we design a variant of MM-GNN by removing this module and use a simple MLP to fuse the output of MME module. MM-GNN with Attention-based Moment Adaptor gains the state-of-the-art results on all datasets compared with the simple MLP fusion method, which demonstrates the superiority of the Attention-based Moment Adaptor module.

	\begin{table*}
		\small
		\setlength{\tabcolsep}{2.6pt}
		\caption{Performance (accuracy) on node classification task with different training set ratio on Facebook social networks. }
		\label{tab:res2}
		\begin{tabular}{cccccccccccccccc}
			\toprule
			Dataset & \multicolumn{3}{c}{Northeastern} & \multicolumn{3}{c}{UF} &\multicolumn{3}{c}{GWU} &\multicolumn{3}{c}{UGA}&\multicolumn{3}{c}{Howard}  \\
			\cmidrule(lr){2-4}\cmidrule(lr){5-7}\cmidrule(lr){8-10}\cmidrule(lr){11-13}\cmidrule(lr){14-16}
			Training &$40\%$&$20\%$& $10\%$&$40\%$&$20\%$& $10\%$&$40\%$&$20\%$& $10\%$&$40\%$&$20\%$& $10\%$&$40\%$&$20\%$& $10\%$\\
			\midrule
			GCN  & 90.34 & 90.38 & 90.35 & 83.97 & 83.76 & 83.62 & 85.76 & 85.65 & 85.67 & 85.5 & 85.46 & 85.42 & 91.17 & 90.95 & 90.99   \\
			GAT  & 88.57 & 88.63 & 88.69 & 79.79 & 79.89 & 79.85 & 83.95 & 84.07 & 83.98 & 82.87 & 82.92 & 82.79 & 89.09 & 88.92 & 88.93   \\
			GraphSAGE & 91.91 & 92.22 & 92.18 & 86.26 & 86.34 & \underline{86.51} & 89.31 & \underline{89.45} & \underline{89.15} & 88.21 & \underline{88.09} & \underline{87.77} & 93.21 & 93.29 & 93.17\\
			GIN  & 89.45 & 89.59 & 89.57 & 85.64 & 85.44 & 85.57 & 84.27 & 83.98 & 84.07 & 87.57 & 87.44 & 87.34 & 89.40 & 89.24 & 89.31  \\
			APPNP  & 89.76 & 89.94 & 89.85 & 81.89 & 81.92 & 81.87 & 87.14 & 87.11 & 87.12 & 84.79 & 84.75 & 84.67 & 89.20 & 89.51 & 89.27 \\
			JKNet  & 92.23 & 92.35 & 92.28 & 85.51 & 85.67 & 85.49 & 89.12 & 89.06 & 89.02 & 87.91 & 87.82 & 87.81 & 93.15 & 93.14 & 93.17 \\
			DAGNN  & \underline{92.49} & \underline{92.36} & \underline{92.30} & \underline{86.86} & \underline{86.46} & {85.70} & \underline{89.73} & {89.31} & 89.11 & \underline{88.35} & 87.79 & 87.62 & \underline{93.31} & \underline{93.33} & \underline{93.27}  \\
			MM-GNN & \colorbox{light-gray}{\textbf{93.08}} & \colorbox{light-gray}{\textbf{92.96}} & \colorbox{light-gray}{\textbf{92.73}} & \colorbox{light-gray}{\textbf{88.34}} & \colorbox{light-gray}{\textbf{87.87}} & \colorbox{light-gray}{\textbf{87.38}} & \colorbox{light-gray}{\textbf{90.63}} & \colorbox{light-gray}{\textbf{90.50}} & \colorbox{light-gray}{\textbf{90.14}} & \colorbox{light-gray}{\textbf{89.52}} & \colorbox{light-gray}{\textbf{89.38}} & \colorbox{light-gray}{\textbf{88.91}} & \colorbox{light-gray}{\textbf{94.03}} & \colorbox{light-gray}{\textbf{93.88}} & \colorbox{light-gray}{\textbf{93.76}}\\
			
			\bottomrule
		\end{tabular}
	\end{table*}
	\begin{figure*}[h]
		\begin{minipage}[t]{0.3\linewidth}
			\centering
			\subfloat[Social network]{\includegraphics[width=\linewidth]{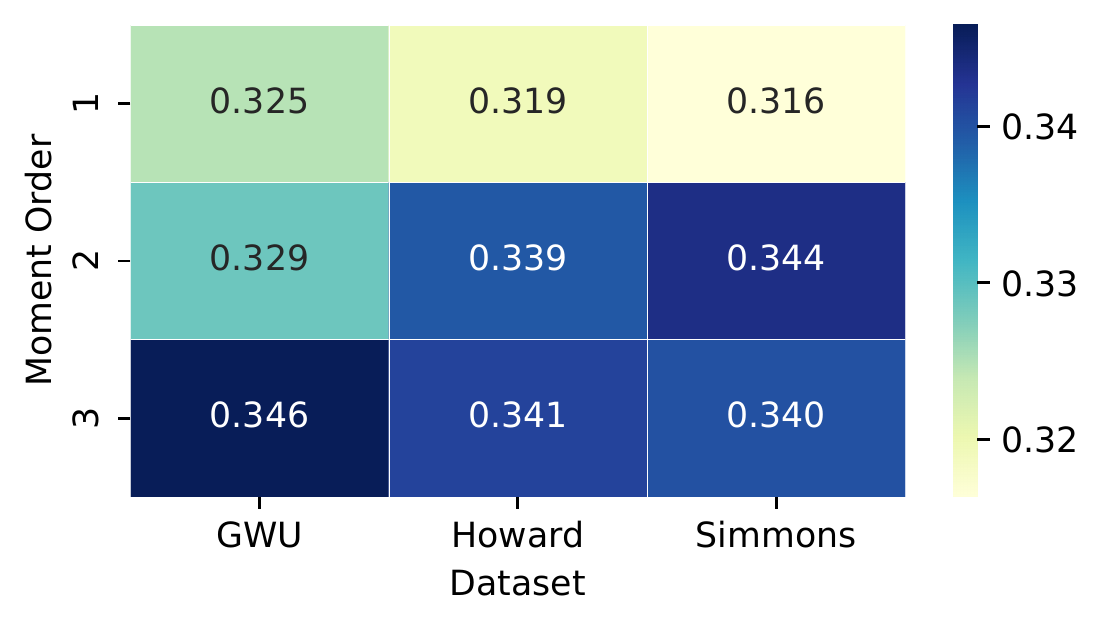}}
		\end{minipage}\quad
		\begin{minipage}[t]{0.3\linewidth}
			\centering
			\subfloat[Citation network]{\includegraphics[width=\linewidth]{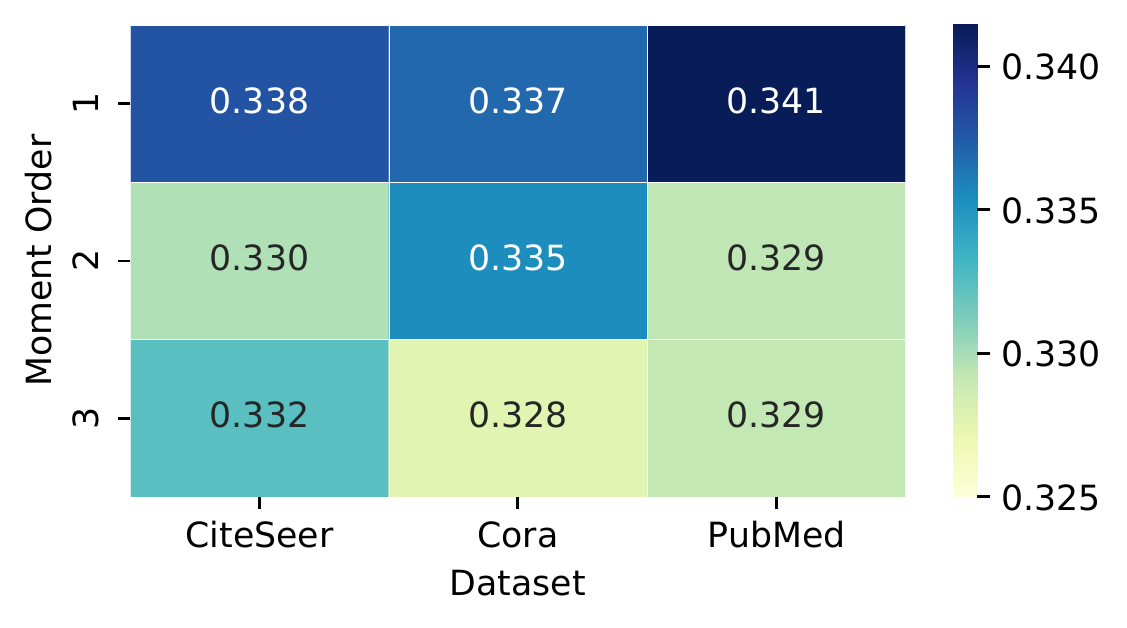}}
		\end{minipage}\quad
		\begin{minipage}[t]{0.3\linewidth}
			\centering
			\subfloat[Webpage/Image network]{\includegraphics[width=\linewidth]{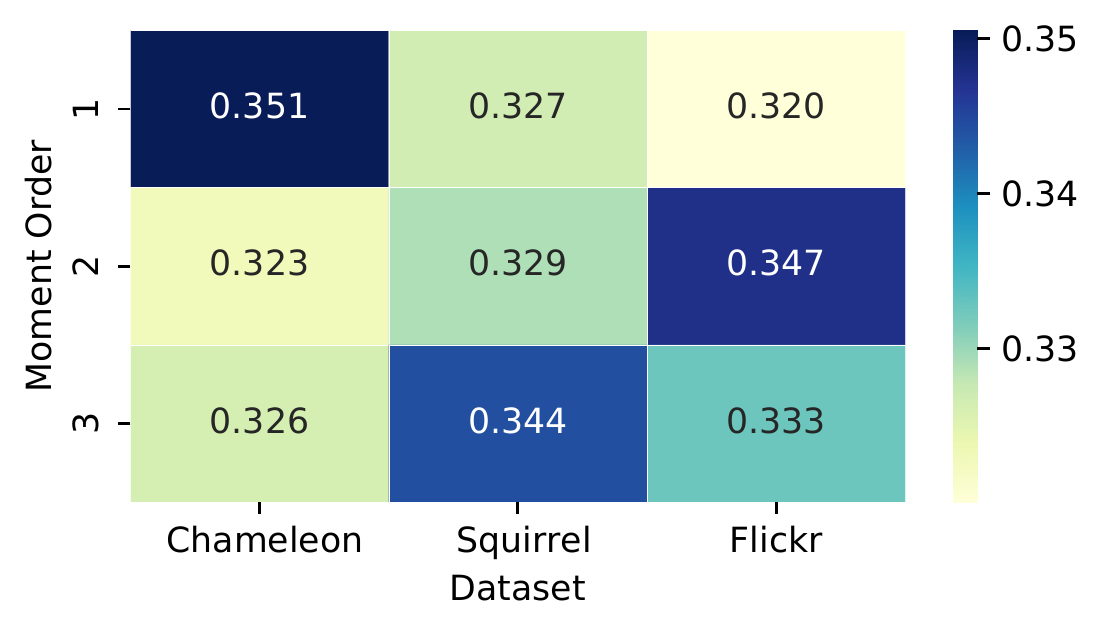}}
		\end{minipage}
		\caption{Average attention value of different moments. The value in each block means the averaged attentions of all nodes and all feature dimensions. And the deeper color means larger attention values learned by MM-GNN.}
		\label{fig:attention_heatmap}
	\end{figure*}

	\subsection{Robustness Analysis}
	In this section, we evaluate the robustness of our MM-GNN by changing the training set ratio and  the key hyper-parameters.
	\subsubsection{Analysis on effects of training set ratio: } We conducted extensive experiments with different training set ratios (10\%, 20\%, 40\%) on the Facebook social networks. And due to space limitations, we only show results on five of these datasets. As shown in Table \ref{tab:res2},  our MM-GNN model has the best performance on each split of the five graphs even when the proportion of training set is small.
	
	\begin{figure}[h]
		\begin{minipage}[t]{0.47\linewidth}
			\centering
			\subfloat[Simmons, GWU, Cora]{\includegraphics[width=\linewidth]{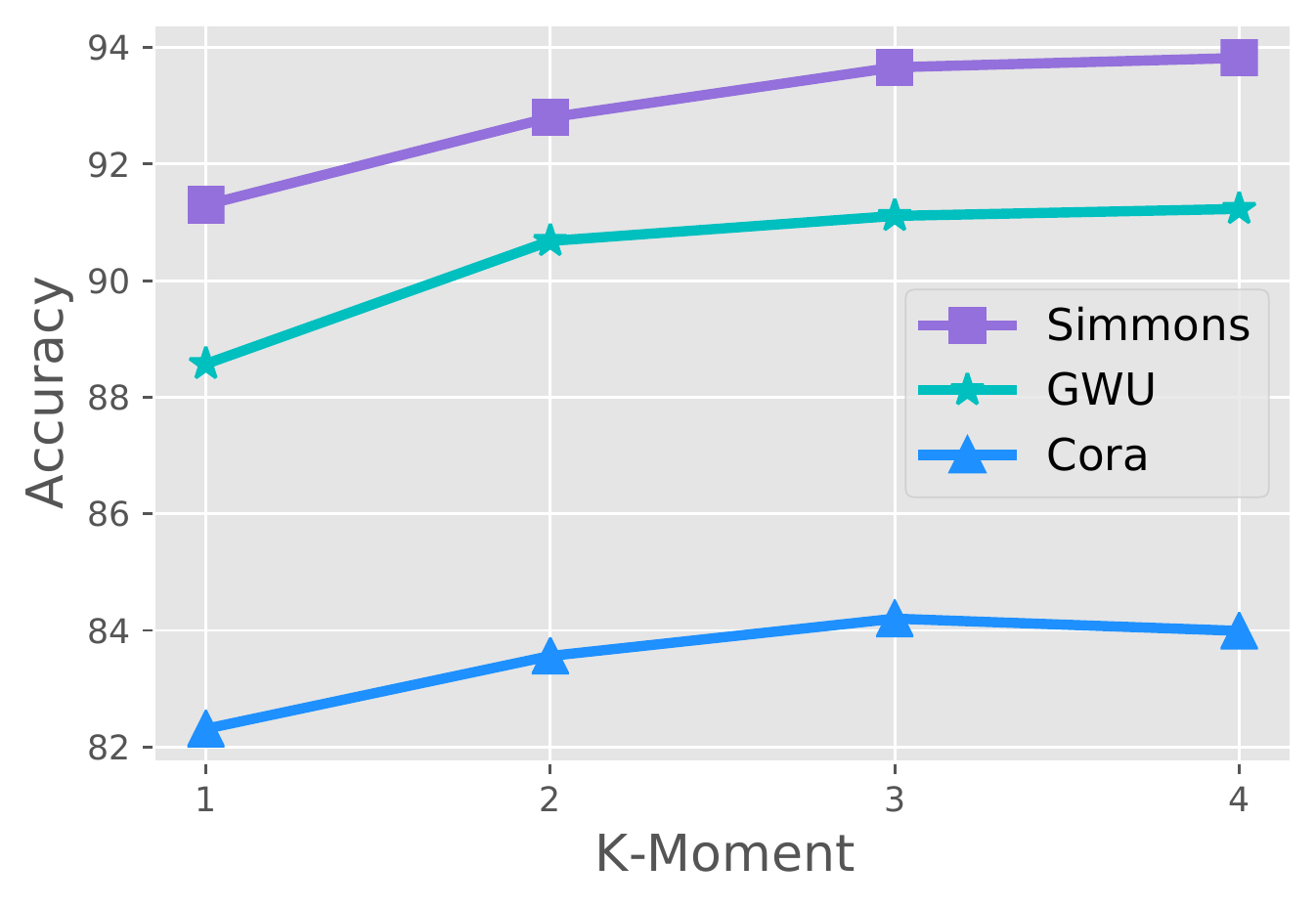}}
		\end{minipage}%
		\quad
		\begin{minipage}[t]{0.47\linewidth}
			\centering
			\subfloat[Chameleon, Squirrel, Flickr]{\includegraphics[width=\linewidth=]{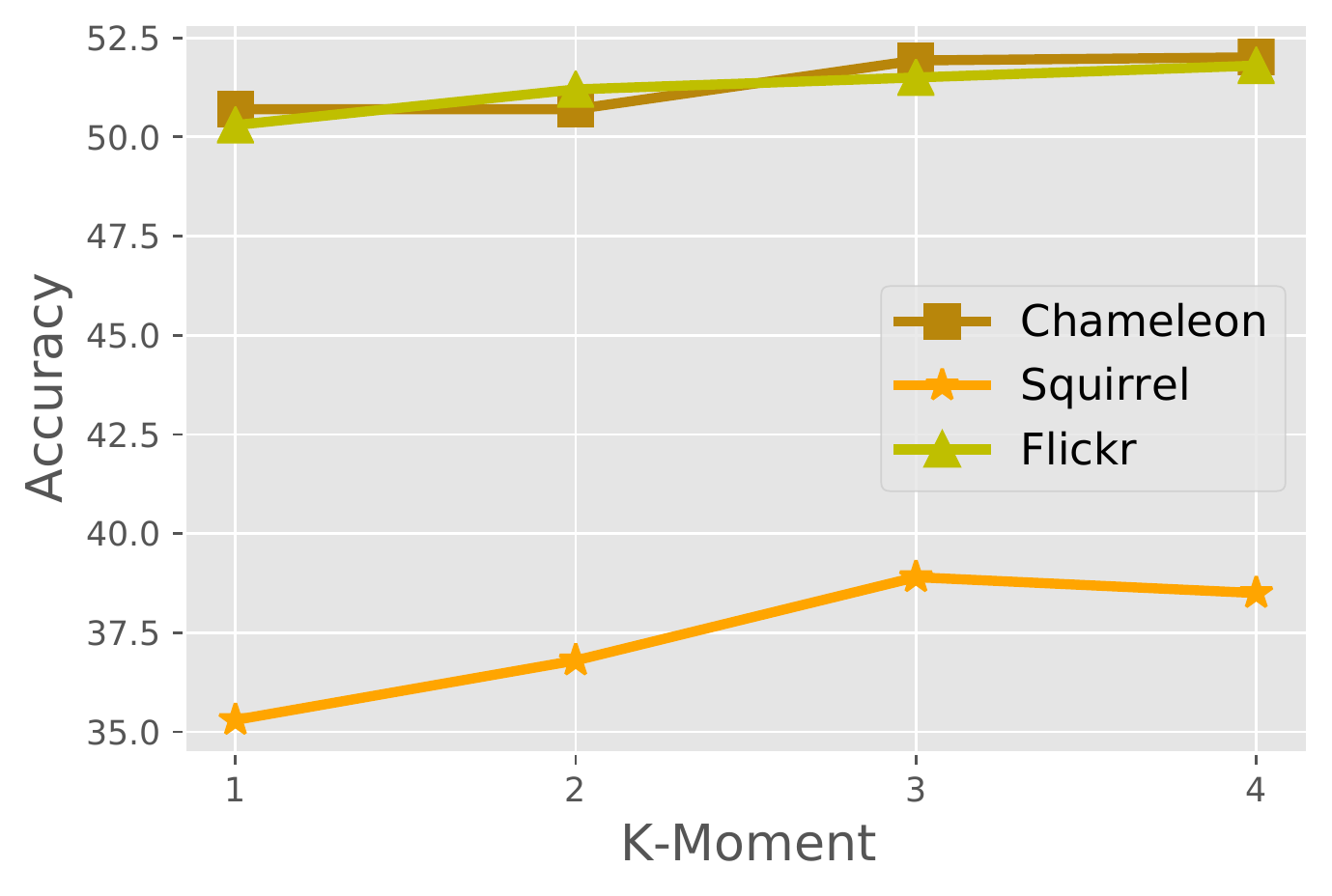}}
		\end{minipage}
		\caption{Hyper-parameter analysis on max ordinal of moments $K$ used in MM-GNN}
		\label{fig:param_analysis_k}
	\end{figure}
	\subsubsection{Analysis on the hyper-parameter $K$: }  	
	As shown in the Fig.~\ref{fig:param_analysis_k}, the performance of MM-GNN is improved with larger $K$ (max ordinal of moments). And when $K$ is no less than 3, the performance of our model is convergent, which matches the data complexity.

	\subsection{Visualizations}
	\label{subsec:visualization}
	The Attention-based Moment Adaptor (AMA) improves the classification performance and also provides interpretability for MM-GNN, which plays an critical role in selecting important orders of moments for each node on the graph. To further explore the interpretability of the learned attention, 
	we analyze their effects through attention visualization. We output the attention value learned by the first layer of MM-GNN for each node and compute the averaged attention of all nodes on all hidden dimensions. As shown in the Fig.~\ref{fig:attention_heatmap}, the average attention values for each moment on different graph dataset present different patterns. MM-GNN output larger weights for higher-order moments on the social networks, while output smaller weights for higher-order moments on citation networks. Moreover, the magnitude of learned attentions are aligned with the performance of the single-moment model shown in Table \ref{tab:ablation}, which further demonstrates the interpretability of MM-GNN.

	\balance
	
    \section{Related work}
	In recent years, Graph Neural Networks (GNNs) have been widely studied on graph representation learning for its expressive power \cite{du2021tabularnet,yao2022trajgat,wang2019tag2gauss,h2gcn,monti2018motifnet,wang2022powerful,bi2022company}. 
	Spectral GNNs \cite{ChebyNet,spectralGNN, GWNN} were first proposed based on graph signal processing theory. Then Kipf et al. proposed GCN \cite{GCN}, which simplifies the convolution operation with low pass filtering on graphs and bridges the gap between spectral GNNs and spatial GNNs. GAT further \cite{GAT} introduced an attention mechanism to GNNs to learn the importance of different neighbors, instead of using equal weights for neighbors. Afterwards, some researchers turned their attention to considering long-range dependencies \cite{dwivedi2021graph, ying2021transformers, abu2019mixhop}. DAGNN \cite{DAGNN} is a deeper GCN model with decoupled transformation and aggregation strategy .
 
	However, most existing GNNs update node representations by aggregating feature from neighbors with single statistic (e.g., mean, max, sum). Though such lower-order statistics  can also reflect certain characteristics of neighbor distribution, they lose the full information of the neighbor's feature distribution. Even for the simple Gaussian distribution which has two degree of freedom (mean, variance), and the information of 2nd-order statistic variance cannot be retained by mean/max/sum (1st-order statistic). Some graph filter based methods \cite{dehmamy2019understanding,corso2020principal} proposed to model graph signals by integrating multiple types of aggregators, without considering the relationship between different signals and neighborhood distribution.  Recently, some works \cite{du2021gbk, ma2022meta, luo2022ada, bi2022make} begin to study the modeling of neighbor distribution on the graph, but they do not consider the information loss caused by single lower-order statistics.

	\section{Conclusion}
	In this paper, we embrace the critical roles of neighbor's feature distribution for modeling graph data. And we introduce the multi-order moments to model neighbor's feature distribution on graphs, fixing the severe limitations of existing GNNs using only single statistic for neighborhood aggregation. Based on the method of moments, we propose a novel Mix-Moment GNN model with adaptive attention mechanism, which includes two components, the Multi-order Moments Embedding (MME) and the Attention-based Moment Adaptor (AMA). MM-GNN use multi-order moments to represent neighbor's feature distribution and select important moments for nodes adaptively. The  consistent improvement of MM-GNN over other state-of-the-art  GNNs on 15 real-world graphs demonstrate the superiority of our method. 
	

	\bibliographystyle{ACM-Reference-Format}
	\bibliography{sample-base}

\end{document}